%% file: unlabeled.tex
\documentclass[12pt]{amsart}
\usepackage{amssymb,amsfonts,euscript,mathrsfs,latexsym,amscd,amsmath}
\usepackage{epsf, color, graphicx, amsthm,amsopn,bbm, microtype, url,tabularx}
\usepackage[dvipsnames]{xcolor}
\usepackage[utf8]{inputenc}
\usepackage[arrow,matrix]{xy}
\usepackage{amssymb}
\usepackage{listings}
\usepackage{hyperref}
\definecolor{darkblue}{RGB}{0,0,160}
\hypersetup{
colorlinks,%
citecolor=black,%
filecolor=black,%
linkcolor=darkblue,%
urlcolor=darkblue
}

\usepackage[
text={440pt,575pt},
headheight=9pt,
centering
]{geometry}

\def\cocoa{{\hbox{\rm C\kern-.13em o\kern-.07em C\kern-.13em o\kern-.15em A}}}

\newcommand{\RR}{\mathbb{R}}

\newcommand{\PP}{\mathbb{P}}

\newcommand\macaulay{\texttt{Macaulay2}}

\newcommand{\<}{\langle}
\renewcommand{\>}{\rangle}

\DeclareMathOperator*{\adj}{adj}

\DeclareMathOperator*{\rk}{rk} 
\DeclareMathOperator*{\spann}{span}
\DeclareMathOperator*{\Sym}{Sym}

\newcommand{\inD}[1][\relax]{\def\argone{#1}\def\temprelax{\relax}
  \ifx\argone\temprelax\right.\else\,\middle|#1\right.{}\fi}

\newtheorem{thm}{Theorem}[section]

\newtheorem{lem}[thm]{Lemma}
\newtheorem{cor}[thm]{Corollary}
\newtheorem{prop}[thm]{Proposition}

\theoremstyle{definition}

\newtheorem{exmp}[thm]{Example}

\newtheorem{rem}[thm]{Remark}

\usepackage[]{algorithm2e}
\usepackage{subcaption}
\definecolor{codegreen}{rgb}{0,0.6,0}
\definecolor{codepurple}{rgb}{0.58,0,0.82}
\definecolor{codered}{RGB}{185,19,6}

 \lstdefinelanguage{myLang}{
   basicstyle=\footnotesize\ttfamily,
   xleftmargin=2em,
   xrightmargin=2em,
   columns=fullflexible,
   keepspaces=true, 
   classoffset=2,
   morekeywords={minors, genericMatrix, matrix, ideal, eliminate, map, degrees, random, numgens, minors, binomial, entries, transpose, submatrix, det},
   keywordstyle={\color{blue}\bfseries},
   classoffset=3,
   morekeywords={for,from, to, do, while},
   keywordstyle={\color{codepurple}\bfseries},
   classoffset=4,
   morekeywords={Height},
   keywordstyle={\color{Emerald}\bfseries},
   classoffset=5,
   morekeywords={QQ, ZZ},
   keywordstyle={\color{codegreen}\bfseries},
   sensitive=false, 
   morecomment=[l]{--}, 
   commentstyle=\color{codered},
   stepnumber=1,
   numbers=left,
   captionpos=b,
   showspaces=false,
   showstringspaces=false,
   morestring=[b]",
   frame=single
}

 \usepackage[arrow,matrix]{xy}
 \usepackage[c]{esvect}

\begin{document}

\title{Algebraic Relations and Triangulation of  Unlabeled Image Points}
\date{July 2017}

\author{Andr\'e Wagner}
\address{Technische Universität Berlin\\ Berlin, Germany}
\urladdr{\url{http://page.math.tu-berlin.de/~wagner}}

\label{chp:Unlabel} 

\begin{abstract}
 In multiview geometry when correspondences among multiple views are unknown the image
 points can be understood as being
 unlabeled. This is a common problem in computer vision. We give a novel approach to handle
 such a situation by regarding unlabeled point configurations as points on the Chow
 variety  $\Sym_m(\PP^2)$.  For two unlabeled points we design an algorithm that solves the
 triangulation problem with unknown correspondences. Further the unlabeled multiview variety $\Sym_m(V_A)$ is studied.
\end{abstract}

\maketitle

\section{Introduction}
In many computer vision applications, the correspondences among views are unknown. Hence the $m$ world points in $\PP^3$ and their images points in $\PP^2$ will be unlabeled. The multiview variety $V_A$  \cite{aholt2011hilbert} is a fundamental object in multiview
geometry, it encodes the relations among $n$ image points  of one world point in $n$ images
taken by $n$ cameras. To study the problem of unlabeled data, we propose to work with the {\em unlabeled multiview variety}. This is
the variety of products of \emph{multiview varieties} \cite{aholt2011hilbert} with unknown correspondences.  An unlabeled point configuration in $(\PP^2)^m$ is a point in the {\em Chow variety} $\Sym_m(\PP^2)$ \cite[\S
8.6]{landsbergtensors}. Algebraically the {\em unlabeled multiview variety} is
the image of the multiview variety under the quotient map  $\bigl((\PP^2)^m \bigr)^n
\rightarrow\bigl( \Sym_m(\PP^2)\bigr)^n$  for the symmetric group $S(m)$ action. Our focus is mostly on two unlabeled points. We design Algorithm \ref{algo:unlabledTriangAlgoComplete} to
triangulate two unlabeled points.

While labeled world and image configurations are points in $(\PP^2)^m $ and $(\PP^2)^m $,
unlabeled image
configurations are points in the {\em Chow varieties} which live as subvarieties in $\Sym_m(\PP^2)$ and $\Sym_m(\PP^3)$.  This is
the variety of ternary forms that are products of $m$ linear forms (cf.~\cite[\S
8.6]{landsbergtensors}), respectively quaternary forms that are products of $m$ linear forms. It is embedded in the space $ \PP^{\binom{m+2}{m}-1} $ of all ternary forms
of degree~$m$.

We start by giving a brief introduction to multiview geometry. A \emph{camera} is a linear map from the three-dimensional
projective space $\PP^3$ to the projective plane $\PP^2$, both over $\RR$.  We represent
$n$ cameras by matrices $A_1,A_2,\ldots,A_n\in \RR^{3\times4}$ of rank~$3$.  A \emph{world
point} $X\in \PP^3$ is mapped by the \emph{perspective relation} of the camera
$A_j\in\RR^{3\times4}$ 
\[
A_jX=\lambda_j u_j,\, \lambda_j \in\RR\setminus \{0\}
\] 
to the \emph{image point} $u_j\in\PP^2$. The kernel of $A_j$ is the {\em focal
point} $f_j \in \PP^3$.  Each image point $u_j=(u_{j0},u_{j1},u_{j2}) \in \PP^2$ of camera $A_j$ has
a line through $f_j$ as its fiber in $\PP^3$. This is the \emph{back-projected line}. On that back-projected line lies the world point $X\in\PP^3$. A camera is called \emph{normalized} if it is of the form
$[I\,|\,c]$ form some vector $c\in\RR^3$, with $I\in \RR^{3\times 3}$ being the identity matrix. 

We assume throughout the paper that the focal points of the $n$ cameras are in {\em general
position}, i.e.~all distinct, no three on a line, and no four on a plane.  Let $\beta_{jk}$ denote
the line in $\PP^3$ spanned by the focal points $f_j$ and $f_k$.  This is the \emph{baseline} of the
camera pair $A_j, A_k$.  The image of the focal point $f_j$ in the $k$th-image plane of the camera
$A_k$ is the \emph{epipole} $e_{kj}$. 

For three cameras $A_j,A_k,A_l$ the plane spanned by their focal points is called
\emph{trifocal plane}.

Computing the point of intersection  of the back-projected lines to reconstruct $X$  is called \emph{triangulation}\cite[\S9.1]{hartley2003multiple}.
 The triangulation can  be based on multiple views and amounts to solve the linear equations
 \begin{equation}\label{eq:multiViewTriangulation}
\quad  B \begin{bmatrix}\,X\\-\lambda_1\\\vdots\\-\lambda_n\end{bmatrix} \, = \, 0
\qquad \hbox{where} \quad
B \,=\, \begin{bmatrix} 
A_1      & u_1    & 0      & \ldots & 0\\ 
A_2      & 0      & u_2    & \ddots & 0\\
\vdots   & \vdots & \ddots & \ddots & \vdots \\ 
A_n      & 0      &\ldots  &      0 & u_n\end{bmatrix} 
\,\in \,\RR^{3n \times (4+n)}.
\end{equation}
For generic data the world point $X$ is represented by the first four entries of any
element in the
kernel of $B$.
In practice this equation system is solved by a
singular value decomposition and similar approaches \cite{hartley1997triangulation}.
Such that a reconstruction of $X$ restricted to only two views is possible the image
points must satisfy a bilinear relation
\[
u_k^TFu_j=0, \,\text{with } F\in \RR^{3\times3}.
\]
The matrix $F$ is called \emph{fundamental matrix} of views $j$ and $k$. 

The \emph{multiview variety} $V_A$ of the camera configuration $A = (A_1,\dots,A_n)$
was defined in \cite{aholt2011hilbert} as the closure of the image of the rational map 
\begin{equation}
  \label{eq:multiMap}
  \begin{matrix}
    \phi_{A}:&  \PP^3  & \dashrightarrow &  \PP^2 \times \PP^2 \times \cdots \times \PP^2, \\
&    X & \mapsto&  (A_1X, A_2X, \ldots, A_nX).
  \end{matrix}
  \end{equation}
The points $(u_1,u_2,\ldots,u_n) \in V_A$ are the 
consistent views in $n$ cameras.
The prime ideal $I_A$ of $V_A$ was determined in \cite[Corollary 2.7]{aholt2011hilbert}.
It is generated by the $\binom{n}{2}$ bilinear polynomials
plus $\binom{n}{3}$  trilinear polynomials. 
See \cite{li2013images} for the natural generalization  
of this variety to higher dimensions.

\begin{rem}\label{rem:binomMulti}
  Let $A_i$ be the matrix with the $i$-row of the $4\times 4$ identiy matrix $I_4$ omitted.
  Then the multiview variety given by the four cameras $A_i, \, i\in[4]$ is isomorphic  up to a change of
 coordinate system to any multiview
  variety with four cameras in linearly general position.
  \end{rem}

We define the \emph{unlabeled multiview variety} $\Sym_m(V_A)$ to be the closure of the image of the rational map
$\zeta_A$
\[ \zeta_{A}: {\Sym}_m(\PP^3) \dashrightarrow ( {\Sym}_m(\PP^2)\bigr)^n.
\] 
Then $\zeta_{A}$ maps an order $m$ symmetric $4\times\ldots\times 4$ tensor to $n$ order $m$
symmetric ${3\times\ldots\times 3}$ tensors. For each world point these tensors have one axis.  This gives rise to the {\em unlabeled multiview
variety} $\,\Sym_m(V_A) $ in $ {\bigl(}\PP^{\binom{m+2}{m}-1}{\bigr)}^{n} $, where $V_A$ is the
multiview variety.  Let $X,Y\in \PP^3$ be two labeled world points. We denote their images
in the $i$-th picture as $u_i,v_i$, 
\[A_iX=\lambda_i u_i, \, A_iY=\mu_i v_i, \, \lambda_i,\mu_i\in \RR\setminus \{0\} .\] 
Then $u_i= (u_{i0},u_{i1},u_{i2})$ and $v_i=(v_{i0},v_{i1},v_{i2})$ are the coordinates of the image
points.
\begin{exmp}
\label{exmp:chow} Let $m=n=2$. The Chow variety $\Sym_2(\PP^2)$ is the hypersurface in $\PP^5$
defined by the determinant of a symmetric $3 \times 3$-matrix $N=(N^{ij})_{i,j\in[3]}$. From the two
image points $u_i,v_i$ we can construct an unlabeled image point on the Chow variety
$\Sym_m(\PP^2)$, this is the symmetric matrix $N=u_iv_i^T+v_iu_i^T$. Up to
the symmetry of coordinates of $N$, this reads as
\[
\begin{matrix} N_1^{00} = 2 u_{10} v_{10}, & N_1^{11} = 2 u_{11} v_{11}, & N_1^{22} = 2 u_{12} v_{12}, \\
N_1^{01} = u_{11} v_{10}+u_{10} v_{11} , & N_1^{02} = u_{12} v_{10}+u_{10} v_{12}, & N_1^{12} = u_{12}
v_{11}+u_{11} v_{12}. \\
\end{matrix}
\] The quotient map $(\PP^2)^2 \rightarrow \Sym_2(\PP^2) \subset \PP^5$ is given by the formulas
above. Similarly, for the two unlabeled images of the the second camera we use
\[
\begin{matrix} N_2^{00} = 2 u_{20} v_{20}, & N_2^{11} = 2 u_{21} v_{21}, & N_2^{22} = 2 u_{22} v_{22}, \\
N_2^{01} = u_{21} v_{20}+ u_{20} v_{21} , & N_2^{02} = u_{22} v_{20}+ u_{20} v_{22}, & N_2^{12} = u_{22}
v_{21}+ u_{21} v_{22}. \\
\end{matrix}
\] 

We compute the image of $V_A \times V_A$ in $\PP^5 \times \PP^5$, denoted $\Sym_2(V_A)$. Its
ideal has seven minimal generators, three of degree $(1,1)$, and one each in degrees $(3,0), (2,1),
(1,2), (0,3)$.  The generators in degrees $(3,0)$ and $(0,3)$ are ${\rm det}(N_1)$ and ${\rm
det}(N_2)$.  The five others depend on the cameras $A_1, A_2$.
\end{exmp}

If $m=2$ we can use symmetric matrices  in $\Sym_2(\RR^4)$ and $\Sym_2(\RR^3)$ to describe the unlabeled multiview variety, even for $n$
larger than two. This is computationally easier to handle.
Then the unlabeled world point configuration of the two
labeled world points $X,Y\in \PP^3$ can be represented by a symmetric $4\times 4$ rank two matrix
$M=XY^T+YX^T$. The unlabeled image point configuration of two labeled image points in the $i$-th
picture $u_i,v_i\in \PP^2$ can be represented by the symmetric $3\times 3$  rank two matrix
$N_i=u_iv_i^T+v_iu_i^T$. 
\begin{rem}Since $X,\, Y,\, u_i,\, v_i$ are points in projective spaces the symmetric rank two matrices $M,\, N_i$ are \emph{only defined up to scale}.
\end{rem}

\begin{figure}
  \includegraphics[width=\textwidth]{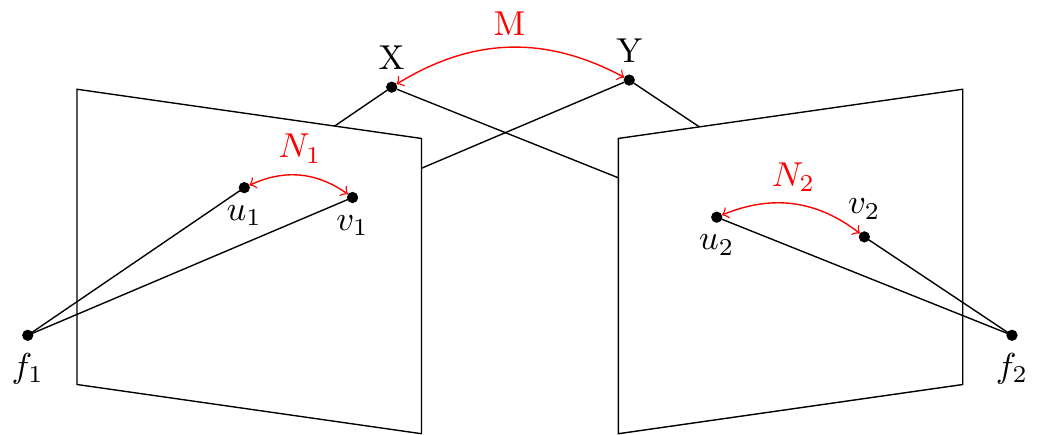}%
   \caption{Unlabeled two-view geometry}
  \label{fig:unlabeledview}
\end{figure} 
 For a symmetric $4\times 4$ rank two matrix $M$ that is suitably generic, $\xi_{A}$ maps the matrix $M$ to $n$ symmetric $3\times 3$ of rank two
matrices. Thus the unlabeled analog of the perspective relation $AX=\lambda u$  for a pinhole camera $A\in \RR^{3\times 4}$ reads as 
\begin{equation}\label{eq:unPerspective}
A(\underbrace{XY^T+YX^T}_{=M})A^T=\lambda (\underbrace{uv^T+vu^T}_{=N}),\, \lambda \in\RR\setminus \{0\}
\end{equation}
The unlabeled multiview map of two points is
\begin{equation}\label{eq:unlabelMapTwo}
\begin{array}{llllll}
\xi_{A}: & \Sym_2(\PP^3) & \hookrightarrow & \Sym_2(\RR^4) & \dashrightarrow & 
\bigl( \Sym_2(\RR^4)\bigr)^n, \\
&  & &  M & \mapsto & (A_1MA_1^T, \ldots ,A_nMA_n^T)  .
\end{array}
\end{equation}
If $m=2$ the unlabeled multiview variety $\Sym_2(V_A)$  is the closure of the image of the rational map $\xi_A$.

We will pay special attention to the algebraic variety $V(\xi_A)$ of the closure of the image of $\xi_A$
in Section \ref{2andChow}.

\section {Relabeling the Unlabeled}
The Chow variety is a tool to
easier handle unlabeled points,  in order
to understand the geometry it is 
sometimes more convenient not to think about unlabeled points as symmetric tensors.
While in the section above we viewed two unlabeled points as one point on the Chow variety $\Sym_m(\PP^2)$,  in this section we will
stick to unlabeled points as points in $\PP^2$ and $\PP^3$. 
Let $\rho$ be the map that takes a labeled point configuration in $(\PP^2)^m$ to its
unlabeled configuration represented as a symmetric tensor, e.g. if $n=2$, $m=2$ then  $\rho$ takes $(u,v)$ and $(v,u)$, both in $\PP^2\times\PP^2$, to the same symmetric matrix $uv^T+vu^T$. The preimage of the unlabeled multiview variety under the map $\rho$ is the union of labeled multiview varieties with the labeling of their image points interchanged.

\begin{figure}
\centering
\begin{subfigure}[b]{.5\textwidth}
  \centering
  \includegraphics[width=.98\linewidth]{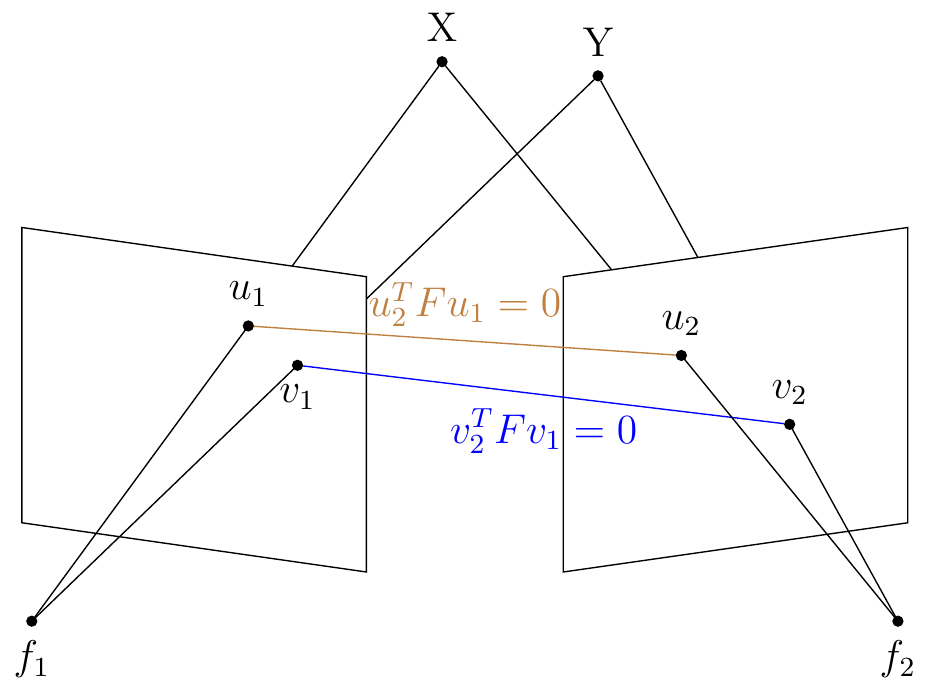}
  \caption{$u_1\thicksim u_2,\,v_1\thicksim v_2$}
  \label{fig:PerumtedLabeled1}
\end{subfigure}%
\begin{subfigure}[b]{.5\textwidth}
  \centering
  \includegraphics[width=.99\linewidth]{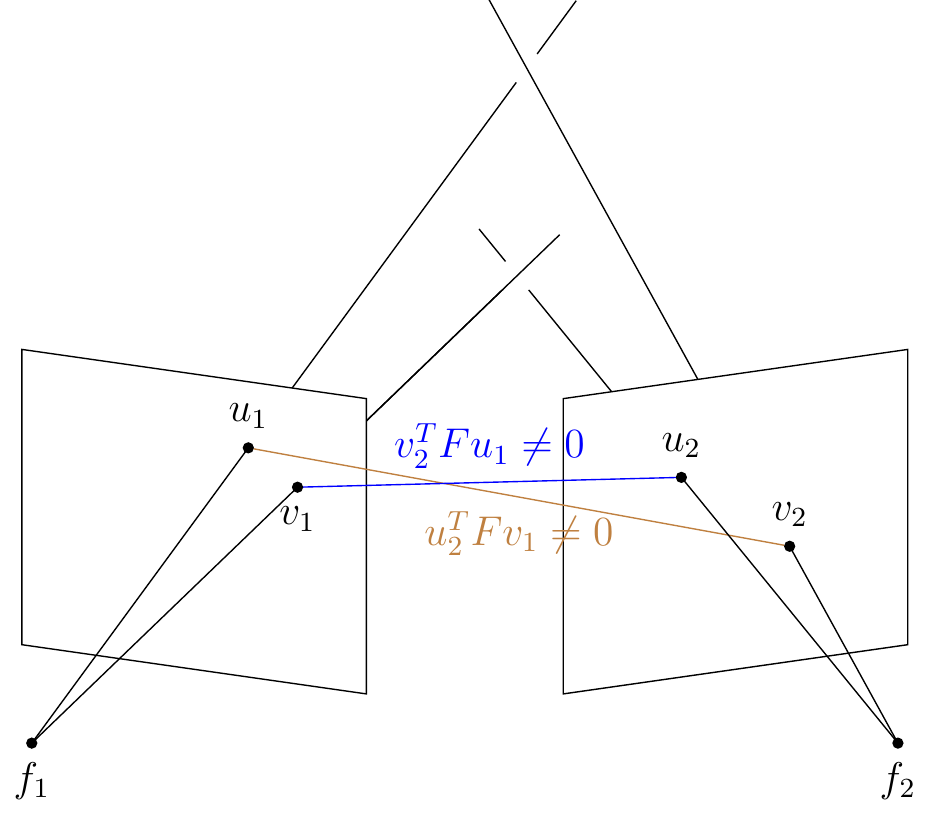}
  \caption{$v_1\thicksim u_2,\,u_1\thicksim v_2$}
  \label{fig:PerumtedLabeled2}
\end{subfigure}
\caption{The two multiview varieties with permuted image point correspondence}
\label{fig:PerumtedLabeled}
\end{figure}
\begin{exmp}\label{labelunlabel} The image points $u_1,v_1,u_2,v_2\in \PP^2$ are on the product of
two two-view varieties if they satisfy
\[u_2^TFu_1=0 \text{ and } v_2^TFv_1=0.\] 
The image points $u_1,v_1,u_2,v_2\in \PP^2$ are on the
 the closure of the preimage under $\rho$ of the unlabeled two-view variety of two points if they satisfy (s. Figure
\ref{fig:PerumtedLabeled})
\begin{equation}\label{eq:twoViewLabelUnlabel} \bigr(u_2^TFu_1=0 \text{ and } v_2^TFv_1=0\bigr)
\text{ or } \bigl(v_2^TFu_1=0 \text{ and }
u_2^TFv_1=0\bigr).\end{equation}
This variety is cut out set-theoretically by the ideal
\[\bigl\<(u_2^TFu_1)(v_2^TFu_1),(u_2^TFu_1)(u_2^TFv_1),(v_2^TFv_1)(v_2^TFu_1),(v_2^TFv_1)(u_2^TFv_1)\bigr\>.\]
\end{exmp}

The approach of Example \ref{labelunlabel} extends to more pictures and more unlabeled
world points. The idea here is to permute the labeling of the image points in each
picture. Each permutation of image points gives a new product of multiview varieties with
permuted coordinates. The union over all these product varieties then is the unlabeled
multiview variety.

We construct a graph with all the image points as vertices, see Figure \ref{fig:permutedImage}.  Let
$G$ be the $n$-partite complete graph on $n\cdot m$ vertices, where each partition of $G$ has
cardinality $m$. Let $K\subset G$ be a perfect $n$-dimensional matching and by $\mathcal K$ we denote
the set of all such perfect $n$-dimensional matchings on $G$. For a fixed $K$ we denote
$P_i(j)$, that is an
image point $u_j^*$ in the $j$-th picture, as the vertex in the $j$-th partition of $G$ in the
$i$-th path $P_i$ of $K$, with $i\in[n],\, j\in[m]$.

In this section we have not been working with unlabeled points in $\Sym_m(\PP^2)$ but with labeled points $(\PP^2)^m$ and their orbits under group action on the image points of the symmetric group. The following theorem illustrates the relation between these two approaches.

\begin{prop}\label{prop:alternative} 
 The closure of the preimage under $\rho$ of the unlabeled multiview variety  is the union of products of multiview varieties, with
permuted image points correspondences
\[ \bigcup_{K\in \mathcal K}\Bigl(V_A\bigl({P_1(1)},{P_2(1)},\ldots,{P_n(1)}\bigr)\times \ldots \times
V_A\bigl({P_1(m)},{P_2(m)},\ldots,{P_n(m)}\bigr)\Bigr).
\]
\end{prop}
\begin{proof} The usual multiview variety knows the labeling of the images points, however the
unlabeled does not. Thus the unlabeled variety is nothing but the union of all multiview varieties
with permuted image point correspondences. A permuted image point correspondence is an
$n$-dimensional matching through the $n$ images, as depicted in Figure \ref{fig:permutedImage}.
Taking the union over all such possible correspondences yields the preimage of the the unlabeled multiview variety.
\end{proof}
\begin{figure}[th] \center
  \includegraphics[width=.6\textwidth]{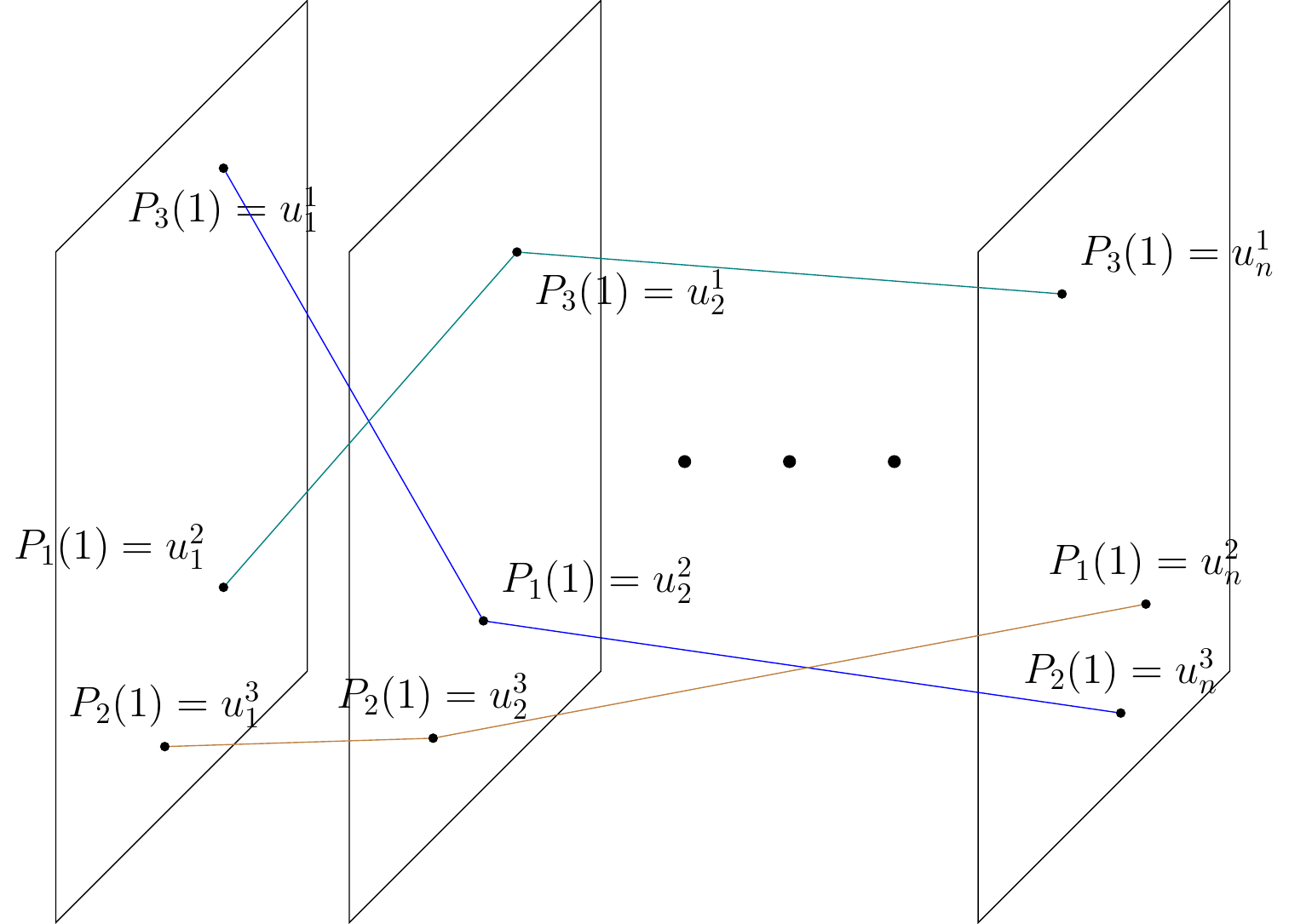}
   \caption{Permuted image point correspondence graph}
  \label{fig:permutedImage}
\end{figure}

\begin{rem}
The multiview variety is cut out set-theoretically by the bilinear constraints, if the
cameras are in linearly general position \cite{heyden1997algebraic}. Thus the unlabeled multiview
variety is cut out by the products of bilinear constraints.
\end{rem}

The usual two-view triangulation does result in either zero solutions ($u_1$, $u_2$ are not on the two-view
variety), one solutions ($u_1$, $u_2$ are on the two-view variety) or a one dimensional linear subspace
 of solutions ($u_1$, $u_2$ are the epipoles). The situation is a little different if we forget the
labeling of the image points. Anyway, the triangulation problem still amounts in intersecting the back-projected
lines. The reconstruction of the original world points is called \emph{unlabeled triangulation}. In the following its ambiguities  are
studied. The unlabeled triangulation is \emph{ambiguous} if there are two or more unlabeled world
point configurations that project down to the same unlabeled image point configuration.
\begin{figure}[b]
\center
  \includegraphics[width=.7\textwidth]{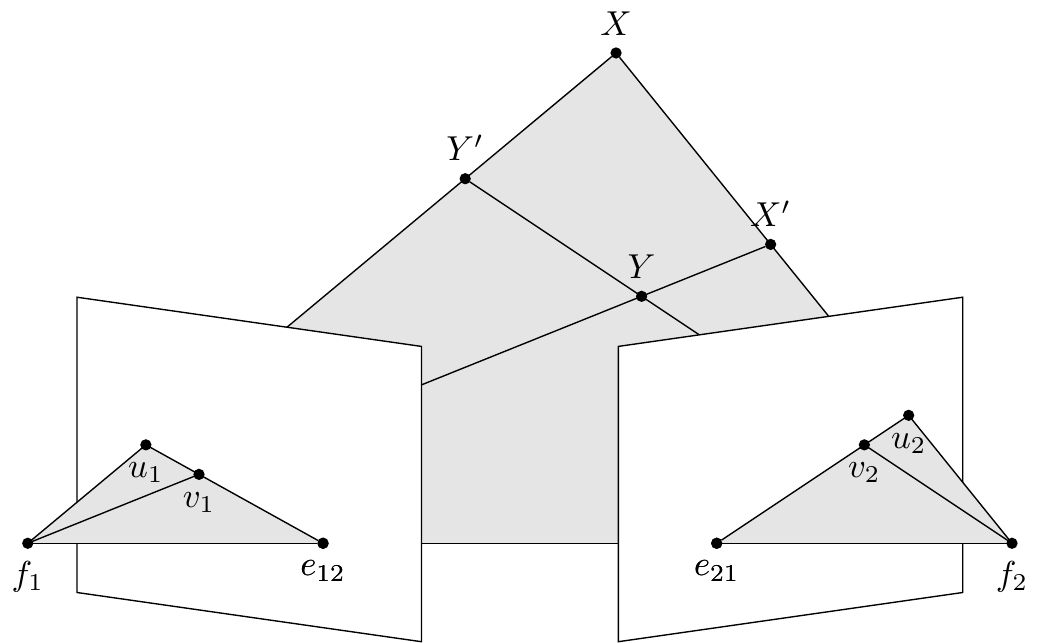}%
   \caption{Ambiguous unlabeled two-view triangulation}
  \label{fig:twoTriangsUnlabel}
\end{figure}
\begin{prop}\label{prop:twoSolutions} Let n=m=2. The unlabeled triangulation is not unique, if and
only if the image points are on the unlabeled multiview variety and satisfy the two conditions
\[
\begin{array}{cc}|e_{12},u_1,v_1|=0, & |e_{21},u_2,v_2|=0\end{array}.
\] 
If the image points
$u_1,u_2,v_1,v_2$ are distinct from the epipoles $e_1,e_2$, then there are exactly two possible
unlabeled world point pairs $(X,Y)$ and $(X',Y')$ that are possible reconstructions.
\end{prop}
\begin{proof} By construction the arrangement of the four back-projected lines of $u_1,u_2,v_1,v_2$ have
the focal points $f_1$ and $f_2$ as intersection points. Such that the triangulation is ambiguous
the four back-projected lines must have at least four additional intersections. Thus the back-projected
lines need to be coplanar by the Veblen-Young axiom.  Because $u_1,u_2,v_1,v_2$ are on the unlabeled multiview variety this is
equivalent to $e_{ij},u_1,v_1$ being collinear, as depicted in Figure \ref{fig:twoTriangsUnlabel}.
Now if $e_{12},u_1,v_1$ are pairwise distinct and $e_{21},u_2,v_2$ are pairwise distinct, then the
back-projected lines have in total six intersections of which four are not the focal points.
\end{proof}

\begin{rem}\label{rem:2coplanar}
Let n=m=2. In world coordinates the ambiguity of the unlabeled triangulation revolves around the world
points $X$, $Y$ and the baseline $\beta_{12}$ through the focal points $f_1,f_2$ being
coplanar. Hence the pencil of planes with the baseline $\beta_{12}$ as its axis encodes the
ambiguities of the two view unlabeled triangulation.
\end{rem}
The ambiguity of the unlabeled triangulation of two views and two unlabeled points extends to the
case $n=2,m\geq 2$.

\begin{cor}\label{cor:twoSolutionsnViews}
If a subset of the world point configuration and the baseline $\beta_{ij}$ are coplanar, then the
unlabeled triangulation is ambiguous.
\end{cor}
\begin{proof}
Follows immediately from Proposition \ref{prop:twoSolutions}. 
\end{proof}
When a subset of $k$ unlabeled world points and the baseline are on one plane, then all
their back-projected lines intersect each other pairwise. There are in total $k^2$ such
intersection points.  To construct an alternative world point configurations, each
back-projected line has to be used exactly one time. Otherwise the resulting image point
configuration does not align with the original image point configuration.

\begin{rem}
 Consider $k$ world points that are coplanar with  the baseline, then there are $k!$ different
solutions to the unlabeled triangulation from two views.
\end{rem}

For generic points on the unlabeled multiview variety the back-projected lines intersect in only two
points and the reconstructing of the unlabeled world point configuration is unique.

When $m>2$ the back-projected lines represent a line arrangement in $\PP^3$, where through
each focal point $m$ lines pass.  Let the \emph{intersection degree} of a point in $\PP^3$ denote
the number of back-projected lines that intersect in that point.  The unlabeled triangulation amounts
in finding points in $\PP^3$ with intersection degree equal to $n$ that are distinct from the focal
points. A configuration of world points is a valid triangulation of the image points if it covers all back-projected lines.

\begin{figure}[!h]
\center
  \includegraphics[width=.8\textwidth]{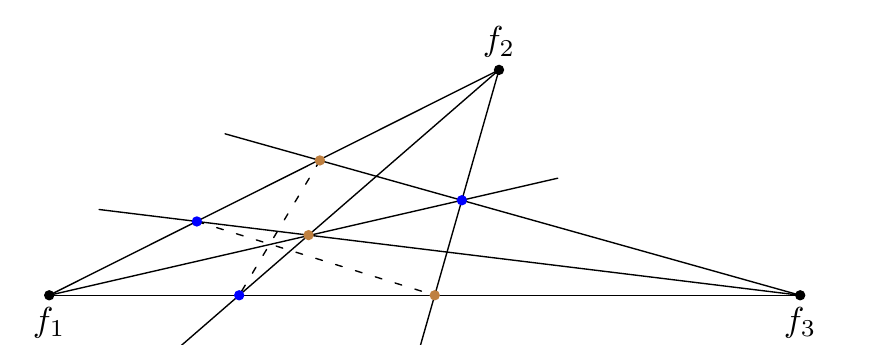}%
   \caption{The blue and brown vertices describe  different world point configurations that project to the same unlabeled image point configuration. The dashed lines denote the missing incidences in Pappus's hexagon theorem.}
  \label{fig:pappos}
\end{figure}
In general classifying the ambiguities of the unlabeled triangulation is a complex problem.  

For three views $n=3$ and three unlabeled points one can employ
Pappus's hexagon theorem to construct an ambiguous unlabeled triangulation. In this case a  configuration that possesses an ambiguous  unlabeled triangulation must lie on the trifocal plane.  Figure \ref{fig:pappos} depicts such an situation on the trifocal plane.

\section{Two Unlabeled Points}
\label{2andChow}

In this section we will only consider the case of two (m=2) unlabeled world points.
The complicated part of the unlabeled multiview map of Equation \ref{eq:unlabelMapTwo} is the rank
two constraint on the symmetric matrix $M$. However if $M$ has rank two we can be sure that its projections $N_i$ have rank two.

\begin{lem}\label{lem:rankSymMat}
The matrices $N_i$ have rank at most two if $M$ has rank two. If $M=XY^T+YX^T$, then $N$ can be written as $N_i=uv^T+vu^T$ for some $u,v\in\PP^2$.  
\end{lem}
\begin{proof}
If $M$ has rank two, then  $\rk(N_i)=\rk(A_iMA_i^T)\leq 2$. The second statement is obtained from the fact \[A_i(XY^T+YX^T)A_i^T=A_iX(YA_i)^T+A_iY(XA_i)^T=\lambda_iuv^T+vu^T=\lambda N_i.\]
\end{proof}
\begin{rem}
The reverse of the second statement of Lemma \ref{lem:rankSymMat} does not hold $N_i=uv^T+vu^T\, \forall \, i\nRightarrow M=XY^T+YX^T$
\end{rem}

We first analyze  $\Sym_2(V_A)$ by dropping the rank two constraints in $\xi_A$. This 
gives us the rational map
\[
\begin{array}{llll}\label{eq:dropedConstraint} \theta_{A}: & \Sym_2(\RR^4) & \dashrightarrow &
\bigl(\Sym_2(\RR^3)\bigr)^n, \\ & M & \mapsto & (A_1MA_1^T, \ldots ,A_nMA_n^T) .
\end{array}
\] 
We will denote the ideal and the variety of the closure of the image of as $I_\theta$ and
$V_\theta$. This variety is a relaxation of the unlabeled multiview variety $\Sym_2(V_A)$.
The map $\theta_A$ is linear in the entries of the symmetric matrix $M$, thus by
\emph{row-wise vectorizing} the upper triangular entries of $M\in\Sym_2(\RR^4)$ to
$\vv{M}\in\RR^{10}$ the map $\theta_{A}$ can be rewritten in its standard form of linear
equations, with the coefficient matrices $\widetilde{A}_i$.  In the coordinates of
$A_i=(a_{jk})_{j\in[3],k\in[4]}$ the coefficient matrix $\widetilde{A}_i$ then reads as
\[
  \widetilde{A}_i=\footnotesize\arraycolsep=7pt\def\arraystretch{0.7}\begin{pmatrix}
  {a}_{11}^{2}&
  {a}_{11} {a}_{21}&
  {a}_{11} {a}_{31}&
  {a}_{21}^{2}&
  {a}_{21} {a}_{31}&
  {a}_{31}^{2}\\
  2 {a}_{11} {a}_{12}&
  {a}_{12} {a}_{21}+{a}_{11} {a}_{22}&
  {a}_{12} {a}_{31}+{a}_{11} {a}_{32}&
  2 {a}_{21} {a}_{22}&
  {a}_{22} {a}_{31}+{a}_{21} {a}_{32}&
  2 {a}_{31} {a}_{32}\\
  2 {a}_{11} {a}_{13}&
  {a}_{13} {a}_{21}+{a}_{11} {a}_{23}&
  {a}_{13} {a}_{31}+{a}_{11} {a}_{33}&
  2 {a}_{21} {a}_{23}&
  {a}_{23} {a}_{31}+{a}_{21} {a}_{33}&
  2 {a}_{31} {a}_{33}\\
  2 {a}_{11} {a}_{14}&
  {a}_{14} {a}_{21}+{a}_{11} {a}_{24}&
  {a}_{14} {a}_{31}+{a}_{11} {a}_{34}&
  2 {a}_{21} {a}_{24}&
  {a}_{24} {a}_{31}+{a}_{21} {a}_{34}&
  2 {a}_{31} {a}_{34}\\
  {a}_{12}^{2}&
  {a}_{12} {a}_{22}&
  {a}_{12} {a}_{32}&
  {a}_{22}^{2}&
  {a}_{22} {a}_{32}&
  {a}_{32}^{2}\\
  2 {a}_{12} {a}_{13}&
  {a}_{13} {a}_{22}+{a}_{12} {a}_{23}&
  {a}_{13} {a}_{32}+{a}_{12} {a}_{33}&
  2 {a}_{22} {a}_{23}&
  {a}_{23} {a}_{32}+{a}_{22} {a}_{33}&
  2 {a}_{32} {a}_{33}\\
  2 {a}_{12} {a}_{14}&
  {a}_{14} {a}_{22}+{a}_{12} {a}_{24}&
  {a}_{14} {a}_{32}+{a}_{12} {a}_{34}&
  2 {a}_{22} {a}_{24}&
  {a}_{24} {a}_{32}+{a}_{22} {a}_{34}&
  2 {a}_{32} {a}_{34}\\
  {a}_{13}^{2}&
  {a}_{13} {a}_{23}&
  {a}_{13} {a}_{33}&
  {a}_{23}^{2}&
  {a}_{23} {a}_{33}&
  {a}_{33}^{2}\\
  2 {a}_{13} {a}_{14}&
  {a}_{14} {a}_{23}+{a}_{13} {a}_{24}&
  {a}_{14} {a}_{33}+{a}_{13} {a}_{34}&
  2 {a}_{23} {a}_{24}&
  {a}_{24} {a}_{33}+{a}_{23} {a}_{34}&
  2 {a}_{33} {a}_{34}\\
  {a}_{14}^{2}&
  {a}_{14} {a}_{24}&
  {a}_{14} {a}_{34}&
  {a}_{24}^{2}&
  {a}_{24} {a}_{34}&
  {a}_{34}^{2}\\
\end{pmatrix}
\]
and the according rational map is
\begin{equation}
\begin{array}{llll}
\widetilde{\theta}_{A}: & \PP^{9} & \dashrightarrow & (\PP^5)^n, \\
 &  \vv{M} & \mapsto & (\widetilde{A}_1\vv M, \ldots ,\widetilde{A}_n\vv M) ,
\end{array}
\end{equation}
This map strongly resembles the conventional multiview map of Equation \ref{eq:multiMap} and can be
understood as a higher dimensional analogue of it. We use the techniques developed in
\cite{li2013images} for a more general setup to describe the closure of the image of
$\widetilde{\theta}_A$ in Corollary \ref{cor:relaxedideal}. But we first  study the unlabeled triangulation. The results we obtain to understand the unlabeled triangulation will come in handy for Corollary \ref{cor:relaxedideal}.

Let $\sigma \subseteq [n]$ with $|\sigma|=k$ and
$\widetilde{A}_\sigma=(\widetilde{A}_{\sigma_1}^T,\ldots,\widetilde{A}_{\sigma_k}^T)$, then define
$B_\sigma$ as
\begin{equation}\label{equ:Bmatrix}B_\sigma=
\begin{bmatrix}
\widetilde{A}_{\sigma_1}&\vv N_{\sigma_1}&&\\
\vdots              &                   &\ddots&\\
\widetilde{A}_{\sigma_k}&                   &      &\vv N_{\sigma_k}
\end{bmatrix}.
\end{equation}

Let the \emph{unlabeled focal point} $f_{ij}$ of the cameras $A_i$ and $A_j$ be  $f_{ij}:=f_if_j^T+f_jf_i^T$.
\begin{prop}\label{prop:rankAtildeTwo} Let $n=m=2$. Let $N_1$ and $N_2$ be two generic points on the
unlabeled two-view variety and denote their unlabeled triangulation $M_\Delta=XY^T+YX^T$. Then $\rk(B_{\{1,2\}})=10$ and there exists $\lambda_1, \, \lambda_2 \in \RR$, such that
\[\spann\bigl((\vv{f_{12}},0,0)^T, (\vv M_\Delta,\lambda_1,\lambda_2)^T\bigr)=\ker(B_{\{1,2\}}).\]

\end{prop}
\begin{proof} Since $N_1$ and $N_2$ are generic points on the unlabeled multiview variety the
triangulation is unique and $M_\Delta=XY^T+YX^T$ is its solution. Hence there are $\lambda_1, \, \lambda_2$,
such that $(\vv M_\Delta,\lambda_1,\lambda_2)$ is in $\ker(B_{\{1,2\}})$.  On the other hand $\widetilde{A}_{i}\vv M=\lambda_i\vv N_i$
is equivalent to $A_i M A_i^T=\lambda_i N_i$, thus inserting $f_1f_2^T+f_2f_1^T$ gives $A_1
\left(f_1f_2^T+f_2f_1^T\right) A_1^T=0$ and inserting gives $A_2\left(f_2f_1^T+f_1f_2^T\right)
A_2^T=0$. Hence for $\lambda_1=\lambda_2=0$ the vector $(\vv f_{12},0,0)$ is in $\ker(B_{\{1,2\}})$.
\end{proof}

\begin{rem}\label{rem:kernelA} By the proof of Proposition \ref{prop:rankAtildeTwo} $\ker(\widetilde{A}_{\{i,j\}})$ is
spanned by $\vv f_{ij}$.
\end{rem}

\begin{prop}\label{prop:rankAtildeThree} Let $n\geq 3$. For a generic choice of cameras
$A_1,\ldots,A_n$ the matrix $\widetilde{A}_\sigma$ has full rank.
\end{prop}

\begin{proof} The unlabeled focal point $f_{ij}$ is the kernel of the submatrix  $\widetilde{A}_{\{i,j\}}$. Since
$f_{ij}$ does not depend on the other cameras it is not in the kernel of any of
the other $\widetilde{A}_k$ for a generic choice of cameras.
\end{proof}
 We can use the matrix $B_\sigma$ to design a triangulation algorithm for the unlabeled multiview variety. We call this algorithm the \emph{unlabeled triangulation algorithm}.
The Proposition \ref{prop:rankAtildeTwo} indicates that the unlabeled triangulation is more
complicated than the usual labeled triangulation. By Propositions \ref{prop:twoSolutions} and Corollary \ref{cor:twoSolutionsnViews} it can have multiple solutions.
Especially the case of two views is more elaborate.  Nonetheless it has a unique solution for generic
points on the unlabeled multiview variety. We need a lemma  before we can state the unlabeled triangulation algorithm.

\begin{lem}\label{lem:twoSubspace}
Consider a two-dimensional linear subspace $V\subset \Sym_2(\RR^4)$ spanned by two symmetric matrices. Suppose that $V$ contains two distinct rank two symmetric matrices $M_1,M_2\in \Sym_2(\RR^4)$, such that $\rk(M_1-M_2)=4$. Then the symmetric matrices $M_1, M_2$ are the only rank two symmetric matrices in $V$. Also the subspace $V$ contains no rank three  symmetric matrix.
\end{lem}
\begin{proof}
We can choose $M_1$ and $M_2$ as basis of $V$. The rank two matrices in $V$ can be
computed as some roots of  $\det(\alpha M_1+(1-\alpha)M_2 )$. Clearly zero and one are
roots of this univariate polynomial. However the roots of $\det(\alpha M_1+(1-\alpha)M_2
)$ are the generalized eigenvalues to the matrix equation $M_1x=\alpha (M_2-M_1)x, \,
x\in\RR^4$. But since $M_1$ has rank two  the generalized eigenvalue $\alpha=0$  has
double algebraic and geometric multiplicity. Then switching the roles of $M_1$ and $M_2$ shows that $\alpha=1$ also has double algebraic and geometric multiplicity.  
\end{proof}
The unlabeled triangulation of two views results in finding the rank two matrices in linear two-dimensional subspaces. That subspace is constructed from $\ker(B_\sigma)$. We can find these matrices by finding the roots of a univariate quartic polynomial, as their determinant needs to vanish. If we only consider generic
points on the unlabeled multiview variety, this polynomial has exactly two real solutions, the unlabeled focal point $f_{ij}$ and the unlabeled triangulation $M_\Delta$ of the unlabeled image points.

\medskip
\begin{algorithm}[H]\label{algo:unlabledTriangAlgo}
 \KwIn{$N_1,N_2$ generic points on the unlabeled multiview variety;}
 \KwOut{$M_\Delta$ unlabeled triangulation of $N_1,N_2$;}
\Begin{
\begin{enumerate}
\item Compute an element $m$ of $\ker{B_\sigma}$, that is not a multiple of $\vv f_{12}$;
\item The first ten entries of $m$ represent a symmetric $4\times 4$ matrix $M$;
\item Compute the quartic polynomial $\det(\alpha M+(1-\alpha)f_{12} )$; 
\item It factors to $\alpha^2(\alpha-c)^2,\, c\in \RR$;
\item Compute $c$; $M_\Delta:=c M+(1-c)f_{12}$;
\end{enumerate}
}
 \caption{Unlabeled triangulation problem for two views}
\end{algorithm}
\begin{proof}
Since $M$ and $f_{12}$ are symmetric with real entries $\det(\alpha M+(1-\alpha)f_{12} )$ only has real roots. By construction $\alpha=0$ is a root. Because $(N_1,N_2)\in \Sym_2(V_A)$ there is a real value $\alpha=c$ such that  $M_\Delta:=c M+(1-c)f_{12}$. Now by Lemma \ref{lem:twoSubspace} these two roots are the only roots of $\det(\alpha M+(1-\alpha)f_{12} )$.
\end{proof}
 Since the roots of the univariate quartic polynomial $\det(\alpha M+(1-\alpha)f_{12} )$  have double multiplicity for generic points, we can find the roots by using the \emph{quadratic formula}.

If $n\geq 3$, then the procedure is simple and amounts to computing $\ker{B_\sigma}$, because the kernel
is one dimensional for generic data by Proposition \ref{prop:rankAtildeThree}. The complete triangulation algorithm can be stated as follows:

\medskip
\begin{algorithm}[H]\label{algo:unlabledTriangAlgoComplete}
 \KwIn{$N_1,\ldots,N_n$ generic points on the unlabeled multiview variety;}
 \KwOut{$M_\Delta$ unlabeled triangulation of $N_1,\ldots,N_n$;}
\eIf{n=2}{
use Algorithm  \ref{algo:unlabledTriangAlgo};}
{
\begin{enumerate}
\item Compute the generator $m$ of $\ker{B_\sigma}$;
\item The first ten entries of $m$ represent the symmetric $4\times 4$ matrix $M_\Delta$;
\end{enumerate}
}
 \caption{Unlabeled triangulation problem}
\end{algorithm}
\begin{proof}
The correctness of Algorithm \ref{algo:unlabledTriangAlgoComplete} follows from the correctness of Algorithm \ref{algo:unlabledTriangAlgo}.
\end{proof}
In some cases one might not want to keep on working with symmetric matrices, but
actually reconstruct the original world points. This amounts to find a decomposition of
$M^\Delta$ into $XY^T+YX^T$. Finding such a decomposition is equivalent to splitting a
degenerate quadrics $\PP^3$ into two planes. To do this, the approach studied  in
\cite{richter2011perspectives} for quadrics in $\PP^2$ generalizes to our case, if one replaces the
adjoint matrix $\adj(M^\Delta)$ by the $6\times 6$ matrix of signed 2-minors of
$M^\Delta$.

We go back to studying the relaxed ideal $I_{{\theta}}$.
 With the help of Proposition \ref{prop:rankAtildeTwo} we are able to determine the
generators of the prime ideal $I_{{\theta}}$ of the image of $\theta_{A}$.

\begin{cor}\label{cor:relaxedideal} 
Let $\sigma\subseteq [n]$. Then the
ideal $I_{{\theta}}$ of the image of $\theta_{A}$ is generated by the maximal minors of $B_\sigma$
(for all $3\leq|\sigma|\leq 10$), and the 11-minors of $B_\sigma$ (for all $|\sigma|=2$).
\end{cor}
\begin{proof} By \cite[\S 2.2]{li2013images} the ideal is generated by the $\rk(\widetilde{A}_\sigma)+|\sigma|$
minors of $B_\sigma$.  If $|\sigma|=2$ then by the proof of Proposition \ref{prop:rankAtildeTwo}
the matrix 
$\widetilde{A}_\sigma$ has rank nine.  If $3\leq|\sigma|\leq 10$ then $\widetilde{A}_\sigma$ has full rank by
Proposition \ref{prop:rankAtildeThree}. Further maximal minors of matrices $B_\sigma$ with
$|\sigma|\geq 11$ are monomial multiples of the other maximal minors. We obtain a generating set by
just computing the minors of correct size.
\end{proof}

\begin{rem}
Let n=2. Consider the  ideal generated by the 11-minors and  the ideal generated by the three degree $(1, 1)$ polynomials of Example \ref{exmp:chow}. These two ideals are equal.
\end{rem}

Clearly the relaxed ideal $I_{\theta}$ is not a good relaxation of $\Sym_2(V_A)$, but it can easily be tightened by enforcing the singularity condition $\det(N_i)=0$. We define the ideal 
\[
I'_\theta:=I_{\theta}+\langle det(N_1),\ldots,det(N_n)\rangle.
\]
Still it is possible to construct rank two matrices on the variety $V(I'_\theta)$ that are not on $\Sym_2(V_A)$.
Let be $N_1=e_{12}v^T+ve_{12}^T, v\in \PP^2$. Then $\rk(B)\leq 10$ independent of $N_2$, thus requiring that the $11\times 11$ minors of $B$ vanish (s. Corollary
\ref{cor:relaxedideal}) does not impose any constraints on $N_2$.  Remember that Equation 
\ref{eq:twoViewLabelUnlabel} gives a description of the preimage of the unlabeled two-view variety of
two points. However
\[ 
\bigl(V(e_{12}^TFu_2)\times V(v_1^TFv_2)\bigr )\cup\bigl(V(e_{12}^TFv_2)\times V(v_1^TFu_2)\bigr)
\] 
gives constraints on $v_1,v_2,u_2$, namely ($v_2^TFu_1=0 \text{ or } v_2^TFv_1=0$). Thus not all
points of
\begin{equation}\label{eq:21isomorph} 
\bigl \{(e_{12}v_1^T+v_1e_{12}^T,N_2):v\in \PP^2,N_2\in {\Sym}_3(\RR)\bigr\}
\end{equation}
 are on the unlabeled multiview variety  $\Sym_2(V_A) $, but they are on $V(I'_\theta)$.

\begin{prop}\label{prop:gens2} 
For n=2, the unlabeled multiview variety is minimally generated by 7
polynomials in 12 variables. It is Cohen-Macaulay and its Betti table is given in Table
\ref{tab:n2Unlabeled}.
\end{prop}

\begin{proof} 
By Remark 4.4 of \cite{aholt2011hilbert} the toric set-up of Remark \ref{rem:binomMulti}
is universal in the sense that every multiview variety with cameras in linearly general
position is isomorphic to the multiview variety of Remark \ref{rem:binomMulti}. Since
the unlabeled multiview variety with cameras in linearly general position is the union of
multiview varieties with correspondences interchanges, any unlabeled multiview variety is
isomorphic to the unlabeled multiview variety with the camera matrices of Remark
\ref{rem:binomMulti}. For two views we only chose the camera matrices $A_1$ and $A_2$ of
Remark \ref{rem:binomMulti} and then ran the code of Listing \ref{algo2} in \macaulay \cite{M2}
with these cameras as input.
\end{proof} 

\begin{table}[b]
\begin{small}
\begin{verbatim}
                                   0 1  2 3 4
                            total: 1 7 11 8 3
                                0: 1 .  . . .
                                1: . 3  2 . .
                                2: . 4  6 . .
                                3: . .  . 2 .
                                4: . .  3 6 3
\end{verbatim}
\end{small}
\smallskip
\caption{\label{tab:n2Unlabeled} Betti numbers for the unlabeled  multiview ideal with $n=2$.}
\end{table}

\begin{rem}
For cameras in linearly general position we have the freedom of choice of a world
coordinate system in $\PP^3$ and coordinate systems in the images in $\PP^2$ without
changing the unlabeled multiview variety up to isomorphism. In that sense we can freely
choose five cameras with focal points in linearly general position to compute an
isomorphic description of the unlabeled multiview variety. 
\end{rem}

As we see in Example \ref{exmp:chow} the generators of the unlabeled two-view variety of two
points ($n=2$, $m=2$) fall  into three different classes according to their multidegree,  namely the ones
with multidegree $(1,1)$,  the ones with degree type $(2,1)$ and the ones with degree  type $(3,0)$. The ones of degree type $(3,0)$ are the determinants of $N_1$ and $N_2$.
\begin{prop}\label{prop:gensFund2} 
 The ideal of  three degree $(1,1)$
equations of the unlabeled two-view variety is a subset of the ideal generated by the
entries of $N_2FN_1$.
\end{prop}

\begin{proof} The statement has been checked for sufficiently many random choices of camera pairs with
\macaulay \cite{M2}, , see Section \ref{compu2} for details.
\end{proof} 

We currently have no way to construct the generators with degree type $(2,1)$. However
we do believe that they relate to the variety of Equation \ref{eq:21isomorph}.

The following statements have been derived with \macaulay \cite{M2}. They are concerned with the prime ideal of the unlabeled multiview variety of three, four and five views.

\begin{prop}\label{prop:gens3} 
For n=3, the unlabeled multiview variety is minimally generated by 60
polynomials in 18 variables and its Betti table is given in Table \ref{tab:n3Unlabeled}.
\end{prop}
\begin{proof}
Similarly to the proof of Proposition \ref{prop:gens2} we choose special cameras to
compute the unlabeled multiview variety. Here the cameras $A_1,A_2$ and $A_3$ of Remark \ref{rem:binomMulti} were chosen for the computation with \macaulay \cite{M2}.
\end{proof} 

\begin{table}[t]
\begin{small}
\begin{verbatim}
            0  1   2    3    4    5    6    7    8    9  10  11 12 13 
     total: 1 60 468 1580 3071 4765 5715 4741 2808 1257 428 102 15  1 
         0: 1  .   .    .    .    .    .    .    .    .   .   .  .  . 
         1: .  9   6    .    .    .    .    .    .    .   .   .  .  . 
         2: . 20 102  159  145   66   12    .    .    .   .   .  .  . 
         3: . 24 273  932 1242  468   60    .    .    .   .   .  .  . 
         4: .  .  12  123  609 2116 2709 1800  657  123  12   .  .  . 
         5: .  7  75  366 1075 2115 2934 2941 2151 1134 416 102 15  1 
\end{verbatim}
\end{small}
\smallskip
\caption{\label{tab:n3Unlabeled} Betti numbers for the unlabeled  multiview ideal with $n=3$.}
\end{table}

\begin{prop}\label{prop:gens4} For n=4, the unlabeled multiview variety is minimally generated by 215
polynomials up to degree six in 24 variables. For n=5, the unlabeled multiview variety is minimally generated by 620
polynomials up to degree six in 30 variables.
\end{prop}
\begin{proof}
Similarly to the proof of Proposition \ref{prop:gens2} we choose special cameras to
compute the unlabeled multiview variety. For $n=4$, the cameras of Remark
\ref{rem:binomMulti} were chosen for the computation with \macaulay \cite{M2}. For $n=5$,
we chose the four cameras of Remark \ref{rem:binomMulti} and additionally the normalized
camera with focal point $e_1=(1,0,0,1)$ for the computation with \macaulay \cite{M2}.
\end{proof}

\section{More than Two Unlabeled Points}
\label{sec:moreUnlabeled}

The study of more than two unlabeled points is quite more elaborate compared to the case
of two points as done in Section \ref{2andChow}. In general, equations describing the Chow
variety are unknown. Also when more than two unlabeled points are present we work with
symmetric tensors instead of symmetric matrices. In Section \ref{2andChow} we strongly
relied on tools from matrix algebra, these are not available or far more complicated for
symmetric tensors. We can still extend some results to the case of $m\geq 3$.

One can construct a point $M$ on the Chow variety $\Sym_m(\PP^3)$ from their unlabeled world
configuration $X_1,\ldots,X_m$ by taking the \emph{sum over all permuted tensors products}
\[ 
M=\sum_{\pi\in S(m)}
\bigl(X_{\pi(1)}\otimes\ldots\otimes X_{\pi(m)}\bigr)\in \RR^{4\times\ldots\times 4},
\]
where $S(\cdot)$ denotes the symmetric group.
The tensor $M$ is symmetric and only defined up to scale. Thus it can be embedded in
$\PP^{\binom{m+3}{m}-1} $. Similarly, the unlabeled image point configurations are
represented by symmetric tensors $N_i\in\RR^{3\times\ldots\times 3}$ defined up to scale,
embedded in $\PP^{\binom{m+2}{m}-1}$. The perspective relation of Equation
\ref{eq:unPerspective} for a pinhole camera $A\in\RR^{3\times4}$ generalizes to
\[
M(\underbrace{A,\ldots,A}_{m \text{ times}})=\lambda N.
\]
We can vectorize this equation above as done in the Section \ref{2andChow}. This yields the linear
equation
\[
\widetilde A \vv M=\lambda \vv N
\]
for some coefficient matrix $\widetilde A$, with  $\binom{m+2}{m}$ rows and $\binom{m+3}{m}$ columns.  Let $\sigma \subseteq [n]$ with $|\sigma|=k$ and
$\widetilde{A}_\sigma=(\widetilde{A}_{\sigma_1}^T,\ldots,\widetilde{A}_{\sigma_k}^T)$ then define
(analogously to the two point case of Equation \ref{equ:Bmatrix}) the matrix $B_\sigma$ as

\[B_\sigma=
\begin{bmatrix} \widetilde{A}_{\sigma_1}&\vv N_{\sigma_k}&&\\ \vdots & &\ddots&\\
\widetilde{A}_{\sigma_k}& & &\vv N_{\sigma_k}
\end{bmatrix}.
\]
This matrix has $k\binom{m+2}{m}$ rows and columns $\binom{m+3}{m}+k$.
The \emph{unlabeled triangulation} $M_\Delta$ can be reconstructed from the matrix $B_\sigma$.
Similar to the two point case one can construct the \emph{unlabeled focal point} of $m$ views, this is
\[f_{[m]}:=\sum_{\pi\in S(m)}\bigl(f_{\pi(1)}\otimes\ldots\otimes f_{\pi(m)}\bigr)\in \RR^{3\times\ldots\times 3}.\] For exactly $m$ views the unlabeled focal point concatenated with $m$ zeros $(\vv f_{[m]},0,\ldots,0)$ is in the kernel of $B_{[m]}$. For $m+1$ views the kernel of $B_{[m+1]}$ is one dimensional and $M_\Delta$ can be computed by linear algebra.

It is of interest to characterize the pictures of $m$ unlabeled points
using $n$ cameras and their ambiguities. Further it would be desirable to
know the prime ideal of $\Sym_m(V_A)$ for any $n$ and~$m$. In Propositions
\ref{prop:gens2}, \ref{prop:gens3} and \ref{prop:gens4} generators of degree at most six appear, thus we believe that the
generators of $\Sym_m(V_A)$ can be constructed from the information obtained by two and three
views.

\subsection{Acknowledgment}
We would like to thank Joe Kileel for  his comments and our debates about this topic during his stay in Berlin. Additionally, we would like to thank Michael Joswig for his guidance.

\section{Computations}
\label{compu2}

We performed several random experiments.
Our hardware was a cluster with Intel Xeon X2630v2 Hexa-Cores (2.8 GHz) and 64GB main memory per node.
The software was \macaulay, version 1.9.2 \cite{M2}.
All computations were single-threaded.

Our result in Proposition \ref{prop:gensFund2} was proved by computations with \macaulay
\cite{M2}. Following standard practice in computational algebraic geometry, we carried out the
computation on many samples in a Zariski dense set of parameters, and then conclude that it holds
generically.  Further instead of using special cameras as we did in the proofs of Propositions
\ref{prop:gens2}, \ref{prop:gens3}, \ref{prop:gens4}, we were also able to compute the unlabeled multiview
variety with random cameras as input.

The computations were repeated several times with random input.  It is not surprising that increasing $n$,
the number of cameras, increases the running times considerably.  In particular using the toric
setup of Remark \ref{rem:binomMulti} is much faster than choosing dense camera matrices $A_i$.

For Proposition \ref{prop:gensFund2} we performed at least 1000 computations to verify its
correctness.

In Listing~\ref{algo2} we show \macaulay code which can be employed to establish
Proposition~\ref{prop:gens2}.

Lines 1--4 define the rings in which the computations take place.
Lines 6--14 produce random camera matrices.
However, our experiments suggest that it suffices to check that the focal points of the cameras are in linearly general position.
The unlabeled multiview map $\theta_A$ from Equation \ref{eq:dropedConstraint} is encoded in lines 17--21.
Lines 13--14 are unlabeled perspective relations (Equation \ref{eq:unPerspective}) and line 26 are the determinantal constraints on the matrices.
The actual computation is the elimination in line~28.
The unlabeled multiview variety is defined in lines 30--31.

\begin{lstlisting}[label=algo2, caption=Compute $\Sym_2(V_A)$ for two cameras]
ImageRing=QQ[a00,a01,a02,a11,a12,a22]**QQ[b00,b01,b02,b11,b12,b22]
WorldRing=QQ[m00,m01,m02,m03,m11,m12,m13,m22,m23,m33]
MultipleRing=QQ[l,k]
S=WorldRing**ImageRing** MultipleRing 

--generate random camera matrices
n  = 2;
AList=0;
while (numgens minors(4,transpose matrix AList)=!=binomial(n*3,4)) do( 
   AList={};
   for i from 0 to n-1 do(
      A_i=random(ZZ^3,ZZ^4,Height=>20);
      AList=AList| entries A_i; )
   );

--create matrices corresponding to unlabeled points
M=genericSymmetricMatrix(S,m00,4);
N_0=genericSymmetricMatrix(S,a00,3); 
N_1=genericSymmetricMatrix(S,b00,3) ;

--unlabeled multiview map
I=ideal( 
    A_0*M*transpose A_0-l*N_0,
    A_1*M*transpose A_1-k*N_1,
    l*k-1 ,
    det(N_0),det(N_1))+minors(3,M);

time I = eliminate({m00,m01,m02,m03,m11,m12,m13,m22,m23,m33,l,k},I)

F = map(ImageRing,S);
J = F(I);
\end{lstlisting}

 \bibliographystyle{amsplain} 
\bibliography{example}
\end{document}